\documentclass[12pt,reqno]{amsart}
\usepackage[utf8]{inputenc}

\usepackage{amsmath, amsthm, amscd, amsfonts, amssymb, graphicx, color}
\usepackage{latexsym}
\usepackage{amscd}
\usepackage{amsthm}
\usepackage{amssymb}
\usepackage{enumerate}
\usepackage{hyperref}
\usepackage{dsfont}

\usepackage[nocites,norefs]{refcheck}

\theoremstyle{plain}
\newtheorem{theorem}{Theorem}

\newtheorem{lemma}[theorem]{Lemma}
\newtheorem{prop}[theorem]{Proposition}
\theoremstyle{remark}
\newtheorem{remark}[theorem]{Remark}
\newtheorem*{remark*}{Remark}

\newcommand{\R}{\mathbb{R}}
\newcommand{\N}{\mathbb{N}}

\newcommand{\W}{\mathbb{W}}

\newcommand{\dist}{\operatorname{dist}}
\newcommand{\eps}{\epsilon}
\newcommand{\indicator}{\mathds{1}}

\begin{document}

\title{Negative results for approximation using single layer and multilayer feedforward neural networks}
\author{J.~M.~Almira,  P.~E.~Lopez-de-Teruel, D.~J.~Romero-L\'{o}pez, F.~Voigtlaender}




\begin{abstract}
We prove a negative result for the approximation of functions defined on compact subsets
of $\mathbb{R}^d$ (where $d \geq 2$) using feedforward neural networks
with one hidden layer and arbitrary continuous activation function.
In a nutshell, this result claims the existence of target functions
that are as difficult to approximate using these neural networks as one may want.
We also demonstrate an analogous result (for general $d \in \mathbb{N}$) for neural networks
with an \emph{arbitrary} number of hidden layers, for activation functions that are either rational functions
or continuous splines with finitely many pieces.
\end{abstract}

\maketitle

\markboth{J.~M.~Almira, P.~E.~Lopez-de-Teruel, D.~J.~Romero-L\'{o}pez, F.~Voigtlaender}
         {Negative results for feedforward neural networks}

\section{Introduction}
\label{sec:Introduction}

The standard model of feedforward neural networks with one hidden layer
leads to the problem of approximation of functions $f : \R^d \to \R$
by elements of the set
\[
  \Sigma_n^{\sigma,d}
  = \bigg\{
      \sum_{k=1}^n c_k \, \sigma(w^k\cdot x -b_k)
      \colon
      w^k \in \mathbb{R}^d, c_k, b_k \in \mathbb{R}
    \bigg\},
\]
where $\sigma \in C(\mathbb{R} ; \R)$ is the given activation function of the network,
and $w^k \cdot x = \sum_{i=1}^d w_{i}^k x_i$ is the dot product
of the vectors $w^k = (w_{1}^k, \dots, w_{d}^k)$ and $x = (x_1, \dots, x_d)$.

It is well known that $\bigcup_{n=1}^{\infty} \Sigma_n^{\sigma,d}$ is dense in $C(\mathbb{R}^d)$
(for the topology of uniform convergence on compact subsets of $\mathbb{R}^d$)
if and only if $\sigma$ is not a polynomial; see \cite[Theorem~1]{LLPS}
(also see \cite{HSW} for a related density result).
This means that feedforward networks with a nonpolynomial activation function
can approximate any continuous function and, thus, are good for any learning objective
in the sense that, given a target function $f \in C(\R^d; \R)$, a precision $\varepsilon > 0$,
and any compact $K \subset \R^d$, there exists $n \in \N$ with the property that
an associated feedforward neural network with one hidden layer and $n$ units can be (in principle)
trained to approximate $f$ on $K$ with uniform error smaller than $\varepsilon$.
In other words, we know that for any nonpolynomial activation function $\sigma\in C(\mathbb{R})$,
any compact set $K\subset \mathbb{R}^d$, any $\varepsilon>0$ and any $f \in C(K)$,
there exists $n_0 = n_0 (\varepsilon, f, \sigma) \in \N$ such that
\[
  E(f, \Sigma_n^{\sigma,d})_{C(K)}
  := \inf_{g \in \Sigma_n^{\sigma,d}}
      \| f - g \|_{C(K)}
  \leq \varepsilon
  \qquad \text{for all} \quad n \geq n_0 .
\]

This result only guarantees, however, that the approximation error vanishes as $n \to \infty$;
it does not provide an explicit error bound.
The study of the \emph{expressivity} of neural network architectures via
more explicit error bounds for the approximation of certain classes of functions $\mathcal{F} \subset C(K)$
has a long history.
As especially noteworthy results we mention the work \cite{MhaskarOptimalApproximationShallowNets}
concerning the approximation of $C^n$ functions using shallow networks with certain smooth
activation functions and the work of Barron \cite{BarronL2Approximation,BarronUniformApproximation}
concerning the approximation using shallow sigmoidal networks of functions $f$
whose Fourier transform $\widehat{f}$ satisfies $\int_{\R^d} |\xi| \, |\widehat{f}(\xi)| \, d \xi < \infty$.
An important feature of the latter result is that it avoids the \emph{curse of dimension}
to a significant extent, since the obtained approximation rate is $\mathcal{O}(n^{-1/2})$,
independent of the input dimension $d$.
In particular, for this function class, neural networks strongly outperform
linear approximation schemes \cite{BarronUniformApproximation}.
Building on the techniques developed by Barron, the authors of
\cite{KainenDependenceOnInputDimension,MhaskarTractabilityOfIntegration} describe
a wide range of function classes for which approximation using neural networks
can overcome the curse of dimension.
On the other hand, several examples of highly-oscillating functions such as those studied
in \cite{SRK,T2,KurkovaHighlyVarying,KurkovaProbabilisticLowerBounds} show
that it is sometimes necessary to dramatically
increase the number of units (or the number of layers) of a neural network
if one wants to approximate certain functions well.
Note that \cite{KurkovaHighlyVarying,KurkovaProbabilisticLowerBounds}
consider functions defined on finite subsets of $\R^d$ and study the number of neurons needed
for a good approximation as $d \to \infty$.
This is in contrast to the setting considered in the present paper,
where the functions are defined on infinite sets and the input dimension is kept fixed.

In addition to quantifying the expressivity (or richness) of sets of neural networks
in terms of upper and lower error bounds for function approximation, other important ``richness measures''
have been studied, including the VC dimension \cite{KarpinskiVCDimensionPfaffian} that plays
a crucial role for generalization bounds, and topological complexity measures
like (the sum of) Betti numbers considered in \cite{BianchiniBettiNumbers}.

Yet, none of the results discussed above provide comprehensive information about the
``worst-case'' decay of best approximation errors \emph{for general continuous functions}.
In this paper we demonstrate a negative result which establishes the existence of target functions $f$
that are as difficult to approximate using neural networks with one hidden layer as one may want.
We also demonstrate an analogous result for neural networks
with an \emph{arbitrary} number of hidden layers for some special types of activation functions $\sigma$.
Concretely, in Section~\ref{sec:SingleLayerNegativeResult} we demonstrate that
for any activation function $\sigma \in C(\mathbb{R})$,
for any input dimension $d\geq 2$ and for any compact set $K \subset \R^d$
with non-empty interior and any given sequence of real numbers $\{\varepsilon_n\}_{n \in \N}$
that converges to zero, there exists a continuous function $f \in C(K)$ such that
\[
  E\bigl(f,\Sigma_n^{\sigma,d}\bigr)_{C(K)}
  \geq \varepsilon_n
  \qquad \text{ for all } n \in \N .
\]
We also demonstrate the same type of result for the norms $L^q (K)$ for $q \in [1,\infty)$.

The proofs of these theorems are based on combining a general negative result
in approximation theory demonstrated by Almira and Oikhberg in 2012 \cite{AO1} (see also \cite{AO2})
with information about the asymptotic decay of the distance between a set of Sobolev functions
and the class of all linear combinations of $n$ ridge functions, as derived
by Maiorov in 2010; see \cite{MaiorovRidgeFunctionSobolevDistance}.

It is important to point out that our result for networks with a single hidden layer requires the use
of functions of at least two variables and does not apply in the univariate setting.
This is natural not only because of the method of proof that we use,
which is based on the result by Maiorov \cite{MaiorovRidgeFunctionSobolevDistance}
that is only true for $d \geq 2$, but also because quite recently
Guliyev and Ismailov have shown that \emph{for the case $d = 1$
a general negative result for approximation by feedforward neural networks
with a single hidden layer is impossible}; see \mbox{\cite[Theorems~4.1 and 4.2]{GI}}.
Precisely, they have explicitly constructed an infinitely differentiable
sigmoidal activation function $\sigma$ such that $\Sigma_2^{\sigma, 1}$
is a dense subset of $C(\mathbb{R})$ with the topology of uniform convergence
on compact subsets of $\mathbb{R}$.
For $\Sigma_3^{\sigma,1}$ instead of $\Sigma_2^{\sigma, 1}$ and with a less explicit construction,
the same result has been established earlier in (the proof of) \mbox{\cite[Proposition~1]{MP}}.

On the other hand, as a consequence of Kolmogorov's superposition theorem
(see \mbox{\cite[Chapter~17, Theorem~1.1]{LGM}}), a general negative result for approximation
by feedforward neural networks with several hidden layers and \emph{general} activation function $\sigma$
is impossible, not only for univariate functions ($d = 1$),
but also for multivariate ($d \geq 2$) functions; see \cite[Theorem~4]{MP} for a proof of this claim
and \cite{GI2, K1,K2} for other related results.
Nevertheless, in Section~\ref{sec:MuliLayerNegativeResult} we demonstrate
several negative results for neural networks with several hidden layers,
\emph{for specific choices of activation functions} $\sigma$ and arbitrary input dimensions $d \geq 1$.
Specifically, we prove that if $\sigma$ is either a continuous rational function or a continuous spline
with finitely many pieces, then for any pair of sequences of natural numbers
$\{ r_k \}_{k \in \N}$ and $\{ n_k \}_{k \in \N}$ and any sequence $\{\varepsilon_k\}_{k \in \N}$
that converges to $0$, and for convex compact subsets $K$ of $\mathbb{R}^d$
of cardinality $\#K \geq 2$, there exists a function $f \in C(K)$ such that
\[
  E \big( f,\tau_{r_k,n_k}^{\sigma,d} \big)_{C(K)}
  \geq \varepsilon_k
  \qquad \text{ for all } k \in \N ,
\]
where $\tau_{r,n}^{\sigma,d}$ denotes the set of functions of $d$ real variables
defined by a neural network with activation function $\sigma$
and at most $r$ layers and $n$ units in each layer.
We also establish a similar result for approximation in $L^q(K)$ for sets
$K \subset \R^d$ with nonempty interior.

It is important to point out that the results of Section~\ref{sec:MuliLayerNegativeResult}
apply for all values of $d \geq 1$ and in particular to neural networks with activation functions
\[
  \sigma(t)
  = \mathrm{ReLU}(t)
  = \begin{cases}
      0 & t <    0 \\
      t & t \geq 0
    \end{cases}
  \qquad \text{and} \qquad
  \sigma(t)
  = \mathrm{Hard\,Tanh}(t)
  = \begin{cases}
      -1 &         t <   -1 \\
       t & -1 \leq t \leq 1 \\
       1 &         t >    1,
    \end{cases}
\]
which are two of the most commonly used activation functions in machine learning.

\section{A negative result for feedforward neural networks with a single hidden layer}
\label{sec:SingleLayerNegativeResult}

The proof of our main result for networks with a single hidden layer will be based on an
abstract result in approximation theory derived in \cite{AO1}.
To properly state the precise result, we first introduce the relevant
notation and terminology from approximation theory.
Given a Banach space $(X, \|\cdot\|)$, we say that $(X, \{A_n\}_{n \in \N_0})$
is an \emph{approximation scheme}
(or that $\{A_n\}_{n \in \N_0}$ is an \emph{approximation scheme} in $X$)
if $\{A_n\}_{n \in \N_0}$ satisfies the following properties
(see, e.g., \mbox{\cite[Definition~1.3]{AO1}}):
\begin{itemize}
  \item[$(A1)$] \(
                  A_0
                  = \{ 0 \}
                  \subsetneq A_1
                  \subsetneq \cdots
                  \subsetneq A_{n}
                  \subsetneq A_{n+1}
                  \subsetneq \cdots
                \)
                is a nested sequence of subsets of $X$ (with \emph{strict} inclusions).
                \vspace{0.1cm}

  \item[$(A2)$] There exists a map $J : \mathbb{N}_0 \to \mathbb{N}_0$
                (called the \emph{jump function} of the approximation scheme) such that
                $J(n) \geq n$ and $A_n + A_n \subseteq A_{J(n)}$ for all $n \in \mathbb{N}_0$.
                \vspace{0.1cm}

  \item[$(A3)$] $\lambda A_n \subseteq A_n$ for all $n \in \mathbb{N}$
                and all scalars $\lambda \in \R$.
                \vspace{0.1cm}

  \item[$(A4)$] $\bigcup_{n \in \mathbb{N}} A_n$ is dense in $X$.
\end{itemize}

For us, it will be important that the family $\{ R_n^d \}_{n \in \N_0}$, with $R_0^d := \{ 0 \}$
and
\begin{equation}
  R_n^d
  := \left\{
       \sum_{i=1}^n
         g_i(a^i\cdot x)
       \colon
       a^i \in \mathbb{S}^{d-1}, \,
       g_i \in L^2_{\mathrm{loc}}(\R), \,
       i=1,\dots, n
     \right\}
  \quad \text{for $n \in \N$},
  \label{eq:RidgeFunctionSets}
\end{equation}
where $\mathbb{S}^{d-1} = \{ x \in \mathbb{R}^d \colon \|x\|_2^2 = \sum_{j=1}^d |x_j|^2 = 1 \}$,
forms an approximation scheme; see Lemma~\ref{lem:RidgeFunctionsAreApproximationScheme} below.
The elements of the set $R_1^d$ are called \emph{ridge functions} in $d$ variables.
The proof that $\{ R_n^d \}_{n \in \N_0}$ indeed forms an approximation scheme is based
on the following result, taken from \cite[Theorem~1]{MaiorovRidgeFunctionSobolevDistance},
which will also play an important role in the proof of our main theorem.

\begin{theorem}\label{thm:MaiorovSobolevRidgeDistance}(\cite{MaiorovRidgeFunctionSobolevDistance})
  Let $d \in \N_{\geq 2}$, $s \in \N$, $r > 0$, $x_0 \in \R^d$,
  and\footnote{Note that there is a slight typo in \cite[Theorem~1]{MaiorovRidgeFunctionSobolevDistance}:
  In the theorem statement, it is assumed that $1 \leq p \leq q \leq \infty$,
  while the correct assumption (under which the theorem is proven) is $1 \leq q \leq p \leq \infty$.}
  $1 \leq q \leq p \leq \infty$.
  Define $B := B_r (x_0) := \{ x \in \R^d \colon \| x - x_0 \| < r \}$ and
  \begin{equation}
    \W^{s,p} := \Big\{
                  f \in L^p (B)
                  \colon
                  \| f \|_{W^{s,p}(B)}
                  := \| f \|_{L^p} + \sum_{|\alpha| = s} \| \partial^\alpha f \|_{L^p}
                  \leq 1
                \Big\} ,
    \label{eq:SobolevDefinition}
  \end{equation}
  where $\alpha \in \N_0^d$ is a multi-index, and the derivative $\partial^\alpha f$
  is understood in the weak sense.
  Then we have
  \[
    \dist \big( \W^{s,p}, L^q(B) \cap R_n^d \big)_{L^q}
    := \sup_{f \in \W^{s,p}}
         E \big( f, L^q(B) \cap R_n^d \big)_{L^q (B)}
    \asymp n^{-s/(d-1)}
    \qquad \text{for all} \quad n \in \N .
  \]
\end{theorem}

\begin{remark}\label{rem:RidgeFunctionsNotDenst}
  \emph{(i)} Theorem~\ref{thm:MaiorovSobolevRidgeDistance} (applied with $p = q$) shows in particular that
  \emph{$L^q (B) \cap R_n^d$ is not dense in $L^q (B)$} for $B = B_r (x_0)$, since otherwise
  we would have $E(f, L^q (B) \cap R_n^d)_{L^q(B)} = 0$ for all
  $f \in \W^{s,p} \subset L^p(B) = L^q(B)$.

  \medskip{}

  \emph{(ii)} The proof in \cite[End of Section~6]{GMMR} (which works for any $1 \leq q \leq p \leq \infty$)
  shows that even
  \begin{equation}
    \dist \big( \W^{s,p}, C(\overline{B}) \cap R_n^d \big)_{L^q (B)}
    \lesssim n^{-s/(d-1)}
    \qquad \text{for all } n \in \N .
    \label{eq:SobolevApproximationWithContinuousRidgeFunctions}
  \end{equation}
\end{remark}

We prove, for the sake of completeness, the following result:
\begin{prop}\label{uno}
  For any $\sigma \in C(\R; \R)$ and $n,d \in \N$, we have
  \(
    \Sigma_n^{\sigma,d} \subseteq R_n^d .
  \)
\end{prop}

\begin{proof}
  Let $\phi(x) = \sum_{k=1}^n c_k \, \sigma(w^k \cdot x - b_k)$ be any element
  of $\Sigma_n^{\sigma, d}$ and let us define
  \[
    g_k : \R \to \R,
    \quad
    g_k(t)
    = \begin{cases}
        c_k \, \sigma(\|w^k\| \, t - b_k) & \text{if } w^k \neq 0 \\
        c_k \, \sigma(- b_k)              & \text{otherwise }
      \end{cases}
  \]
  and $a^k := w^k / \|w^k\|$ if $w^k \neq 0$,
  while $a^k := e_1 = (1, 0, \dots, 0) \in \mathbb{R}^d$ if $w^k = 0$.
  Then $g_k \in C(\mathbb{R}) \subset L^2_{\mathrm{loc}}(\R)$,
  and $a^k \in \mathbb{S}^{d-1}$ for $k = 1, \dots, n$.
  Moreover,
  \[
    \sum_{k=1}^n
      g_k( a^k \cdot x)
    = \sum_{ \{ k : w^k \neq 0 \}}
        c_k \, \sigma(\|w^k\| a^k \cdot x - b_k)
      + \sum_{ \{ k:w^k = 0 \} }
          c_k \, \sigma(-b_k)
    = \sum_{k=1}^n
        c_k \, \sigma(w^k \cdot x - b_k)
    = \phi(x),
  \]
  which means that $\phi \in R_n^d$.
\end{proof}

Now, we can prove that the family $\{ R_n^d \}_{n \in \N_0}$ of sums of ridge functions
indeed forms an approximation scheme.

\begin{lemma}\label{lem:RidgeFunctionsAreApproximationScheme}
  Let $d \in \N_{\geq 2}$, and let $K \subset \R^d$ be a compact set with non-empty interior.
  Let $X := C(K)$, or $X := L^q(K)$ for some $q \in [1,\infty)$.
  Let $R_0^d := \{ 0 \}$, and for $n \in \N$ let $R_n^d$ as defined
  in Equation~\eqref{eq:RidgeFunctionSets}.
  Then $(X, \{ R_n^d \cap X \}_{n \in \N_0})$ is an approximation scheme with jump function
  $J(n) := 2n$.
\end{lemma}

\begin{proof}
  It is easy to see that $\lambda R_n^d \subseteq R_n^d$ for all $\lambda \in \R$ and $n \in \N_0$,
  so that Property~(A3) is satisfied.
  Next, note that $R_n^d = R_1^d + \dots + R_1^d$, with $n$ summands, which easily shows that
  $(R_n^d \cap X) + (R_n^d \cap X) \subseteq R_{2n}^d \cap X$, so that also (A2) is satisfied.

  Furthermore, if we choose (e.g.) $\sigma : \R \to \R, x \mapsto \max \{ 0, x \}$,
  then Proposition~\ref{uno} shows that
  $\bigcup_{n \in \N} R_n^d \cap X \supseteq \bigcup_{n \in \N} \Sigma_n^{\sigma,d}$,
  where the right-hand side is dense in $\big(C(K), \| \cdot \|_{C(K)} \big)$ by \cite[Theorem~1]{LLPS},
  and hence also dense in $X$ with respect to the norm of $X$.
  Therefore, Property~(A4) is satisfied as well.

  To prove (A1), first note that $R_0^d = \{ 0 \}$ and $R_n^d \cap X \subseteq R_{n+1}^d \cap X$.
  Now, assume towards a contradiction that $R_n^d \cap X = R_{n+1}^d \cap X$ for some $n \in \N_0$.
  We claim that this implies $(X \cap R_n^d) + \Sigma_k^{\sigma,d} \subset X \cap R_n^d$
  for all $k \in \N_0$, where $\Sigma_0^{\sigma,d} := \{ 0 \}$.
  Indeed, for $k = 0$, this is trivial.
  Now, if the claim holds for some $k \in \N_0$, and if $f \in X \cap R_n^d$
  and $g \in \Sigma_{k+1}^{\sigma,d}$, then $g = g_1 + g_2$ for certain
  $g_1 \in \Sigma_1^{\sigma,d}$ and $g_2 \in \Sigma_k^{\sigma,d}$.
  Proposition~\ref{uno} shows that
  ${\Sigma_\ell^{\sigma,d} \subset R_\ell^d \cap C(K) \subset X \cap R_\ell^d}$, so that
  $f + g_1 \in (X \cap R_n^d) + (X \cap R_1^d) \subset X \cap R_{n+1}^d = X \cap R_n^d$.
  By induction, this implies
  \(
    f + g
    = f + g_1 + g_2
    \in (X \cap R_n^d) + \Sigma_k^{\sigma,d}
    \subset X \cap R_n^d .
  \)
  We have thus shown
  \(
    (X \cap R_n^d) + \bigcup_{k = 1}^\infty \Sigma_k^{\sigma,d}
    \subset X \cap R_n^d
    \subset L^1 (K) \cap R_n^d
    ,
  \)
  where the left-hand side is dense in $L^1 (K)$, while the right-hand side is not,
  by Remark~\ref{rem:RidgeFunctionsNotDenst} and since $K$ has non-empty interior.
  This contradiction shows that $X \cap R_n^d \subsetneq X \cap R_{n+1}^d$ for all $n \in \N_0$,
  as needed for Property~(A1).
\end{proof}

As the final preparation for the proof of our main result regarding
feedforward networks with a single hidden layer,
we collect the following abstract result about approximation schemes
from \cite[Theorems~2.2 and 3.4]{AO1}.
See also \cite[Theorem~1.1]{AO2} for a related result.

\begin{theorem}[Almira and Oikhberg, 2012]\label{AlmiraNR}
  Given an approximation scheme $\{ A_n \}_{n \in \N_0}$ in the Banach space $X$,
  the following are \emph{equivalent} claims:
  \begin{itemize}
    \item[$(a)$] For every null-sequence $\{\varepsilon_n\}_{n \in \N} \subset \R$
                 there exists an element $x\in X$ such that
                 \[
                   E(x,A_n)_X
                   = \inf_{a_n\in A_n}
                       \| x - a_n \|_X
                   \geq \varepsilon_n
                   \qquad \text{ for all } n \in \mathbb{N} .
                 \]

    \item[$(b)$] There exists a constant $c > 0$ and an infinite set
                 $\mathbb{J}_0 \subseteq \mathbb{N}$ such that, for all $n \in \mathbb{J}_0$
                 there exists $x_n \in X \setminus \overline{A_n}$ such that
                 \[
                   E(x_n, A_n)_X \leq c \, E(x_n, A_{J(n)})_X
                   \qquad \text{for all } n \in \mathbb{J}_0,
                 \]
                 where $J$ is the jump function in Condition~(A2).
  \end{itemize}
\end{theorem}

\begin{remark*}
  If $(X, \{A_n\}_{n \in \N_0})$ satisfies $(a)$ or $(b)$ of Theorem~\ref{AlmiraNR},
  we say that the approximation scheme satisfies \emph{Shapiro's theorem}.
\end{remark*}

Let us state the main result of this section:

\begin{theorem}\label{thm:ShallowLowerBound}
  Let $\sigma \in C(\mathbb{R}; \R)$.
  Let $d \geq 2$ be a natural number and let $K \subset \mathbb{R}^d$ be a compact set
  with nonempty interior.
  Let either $X = L^q (K)$ for some $q \in [1,\infty)$, or $X = C(K)$.

  For any sequence of real numbers $\{\varepsilon_n\}_{n \in \N}$
  satisfying $\lim_{n \to \infty} \varepsilon_n = 0$,
  there exist a target function $f \in X$ such that
  \[
    E(f, \Sigma_n^{\sigma,d})_{X}
    \geq E(f, X \cap R_n^{d})_{X}
    \geq \varepsilon_n
    \qquad \text{for all } n \in \N .
  \]
\end{theorem}

\begin{proof}
  Since $K$ has nonempty interior, there are $r > 0$ and $x_0 \in K$
  such that $\overline{B} \subset K$, for $B := B_r (x_0) \subset K$.
  If $X = L^q(K)$ set $X_0 := L^q (B)$; if otherwise $X = C(K)$, set $X_0 := C(\overline{B})$
  and $q := \infty$.
  Let $\W^{1,q}$ as defined in Equation~\eqref{eq:SobolevDefinition}.
  We then have $\W^{1,q} \subset X_0$; indeed, for $q < \infty$ this is clear,
  and if $q = \infty$, then the Sobolev embedding theorem
  (see \cite[Section~5.6, Theorem~5]{EvansPDE}) shows that
  $\W^{1,\infty} \subset W^{1,\infty} (B) \subset W^{1, 2d} (B) \subset C(\overline{B}) = X_0$,
  where $W^{k,p}(B)$ denotes the usual Sobolev space.

  Now, by Equation~\eqref{eq:SobolevApproximationWithContinuousRidgeFunctions}
  and Theorem~\ref{thm:MaiorovSobolevRidgeDistance} (both applied for $s = 1$),
  there exist two positive constants $c_0, c_1 > 0$
  (depending only on $d$ and on $r = r(K)$) such that:
  \begin{itemize}
    \item[(i)] For any $f \in \W^{1,q} \subset X_0$, we have
               \[
                 \qquad \quad
                 E(f, X_0 \cap R_n^d)_{X_0}
                 = E(f, X_0 \cap R_n^d)_{L^q(B)}
                 \leq E \big( f, C(\overline{B}) \cap R_n^d \big)_{L^q(B)}
                 \leq c_1 \, n^{-1/(d-1)} .
               \]

    \item[(ii)] For any $m \in \mathbb{N}$ there exists $f_m \in \W^{1,q} \subset X_0$ such that
                \[
                  E(f_m, X_0 \cap R_m^d)_{X_0}
                  \geq E(f_m, L^q (B) \cap R_m^d)_{L^q(B)}
                  \geq c_0 \, m^{-1/(d-1)} .
                \]
  \end{itemize}
  Combining these inequalities for $m = 2n$, $n \in \mathbb{N}$, we have that
  \begin{align*}
    E\bigl(f_{2n}, X_0 \cap R_{2n}^d\bigr)_{X_0}
    & \geq  c_0 \, (2n)^{-1/(d-1)}
      =     c_0 \, 2^{-1/(d-1)} \, n^{-1/(d-1)} \\
    & \geq  2^{-1/(d-1)} c_1^{-1} c_0 \,\,
            E\bigl(f_{2n}, X_0 \cap R_n^d\bigr)_{X_0} .
  \end{align*}
  In particular, Property~(ii) above implies that
  $E(f_{2n}, X_0 \cap R_n^d)_{X_0} \geq E(f_{2n}, X_0 \cap R_{2 n}^d)_{X_0} > 0$
  and hence $f_{2n} \in X_0 \setminus \overline{X_0 \cap R_n^d}$.
  Furthermore,
  \[
    E(f_{2n}, X_0 \cap R_n^d)_{X_0}
    \leq 2^{1/(d-1)} \, c_0^{-1} c_1 \,\,
         E(f_{2n}, X_0 \cap R_{2n}^d)_{X_0}
    \qquad \text{for all } n \in \mathbb{N} ,
  \]
  which shows that the approximation scheme $\bigl(X_0, \{ X_0 \cap R_n^d \}_{n \in \N_0}\bigr)$
  satisfies Condition (b) of Theorem~\ref{AlmiraNR}, with $\mathbb{J}_0 = \N$.

  Thus, by Theorem~\ref{AlmiraNR}, for the given sequence $\varepsilon_n \to 0$,
  there exists a function $f \in X_0$ such that
  \[
    E(f, X_0 \cap R_{n}^d)_{X_0}
    \geq \varepsilon_n
    \qquad \text{for all } n \in \mathbb{N}.
  \]
  In case of $q < \infty$, one can extend $f$ by zero to obtain a function $\widetilde{f} \in X$
  such that $\widetilde{f}|_B = f$.
  If otherwise $q = \infty$, we can use the Tietze extension theorem
  (see \cite[Theorem~4.34]{FollandRA}) to obtain a continuous function
  $\widetilde{f} \in C(K) = X$ such that $\widetilde{f}|_{\overline{B}} = f$.
  In any case, we then see by Proposition~\ref{uno} that
  \[
    E\bigl(\widetilde{f}, \Sigma_n^{\sigma,d}\bigr)_{X}
    \geq E\bigl(\widetilde{f}, X \cap R_n^d\bigr)_X
    \geq E\bigl(f, X_0 \cap R_n^d\bigr)_{X_0}
    \geq \varepsilon_n
    \qquad \text{for all } n \in \N,
  \]
  as desired.
\end{proof}

\section{Negative results for feedforward neural networks  with several hidden layers}
\label{sec:MuliLayerNegativeResult}

In \cite[Theorem~4]{MP} it was proved that,
for a proper choice of the activation function $\sigma$,
which may be chosen to be real analytic and of sigmoidal type,
a feedforward neural network with two hidden layers and a
\emph{fixed finite number of units in each layer}
is enough for uniform approximation of arbitrary continuous functions
on any compact set $K \subset \mathbb{R}^d$.
Concretely, the result was demonstrated for neural networks with $3d$ units
in the first hidden layer and $6d + 3$ units in the second hidden layer.
Moreover, reducing the restrictions on $\sigma$, the same result can be obtained
for a neural network of two hidden layers with $d$ units in the first hidden layer
and $2d + 2$ units in the second hidden layer; see \cite{IsmailovBoundedNumberOfNeurons}.
Finally, Guliyev and Ismailov have recently shown
that there exists an \emph{algorithmically computable} activation function $\sigma$
such that feedforward neural networks with two hidden layers and
with $3d + 2$ (hidden) units in total can uniformly approximate
arbitrary continuous functions on compact subsets of $\mathbb{R}^d$; see \cite{GI2}.

In this section we prove that, for certain natural choices of $\sigma$
(which exclude the pathological activation functions discussed above),
a negative result holds true for neural networks with any number of hidden layers
and units in each layer.
Concretely, consider feedforward neural networks
for which the activation function $\sigma$ is either a continuous rational function
or a continuous piecewise polynomial function with finitely many pieces.
Examples of such activation functions are the well known $\mathrm{ReLU}$ and $\mathrm{Hard\,Tanh}$
activation functions, which are linear spline functions with two and three pieces, respectively.
Note that the use of rational or spline approximation tools
in connection with the study of neural networks is natural
and has been used by several other authors; see, for instance \cite{SRK, T, T2, W, WB}.

In order to prove our negative result, we first collect several facts from the literature.
First, from the findings in \cite[Section~6.3]{AO1} concerning rational functions,
we get the following:

\begin{theorem}\label{rational}
  Let $I = [a, b]$ with $a < b$, and let $\mathcal{R}_n^1(I)$ denote the set of rational functions
  $r(x) = p(x) / q(x)$ with $\max \{ \deg (p), \deg (q) \} \leq n$
  and such that $q$ vanishes nowhere on $I$.
  Let $\{n_i\}_{i \in \N}$ be a sequence of natural numbers.
  Let $A_0 := \{ 0 \}$ and $A_i := \mathcal{R}_{n_i}^1(I)$, for $i \in \N$.
  Finally, let either $X = C(I)$ or $X = L^q(I)$ where $q \in (0, \infty)$.
  Then for any null-sequence $\{ \epsilon_i \}_{i \in \N}$,
  there is a function $f \in X$ satisfying
  \[
    E(f, A_i)_{X} \geq \epsilon_i
    \qquad \text{for all } i \in \N .
  \]
\end{theorem}

\begin{proof}
  In case of $n_i = i$ for all $i \in \N$,
  the result follows from \cite[Item $(1)$ of Theorem~6.9]{AO1}.

  Now, let $\{ n_i \}_{i \in \N}$ be a general sequence of natural numbers.
  Define $k_i := i + \max_{1 \leq j \leq i} n_j$ for $i \in \N$,
  noting that the sequence $\{ k_i \}_{i \in \N}$ is strictly increasing and satisfies
  $k_i \geq n_i$ for all $i \in \N$.
  Define $k_0 := 0$.
  With the null-sequence $\{\epsilon_i\}_{i \in \N}$ given in the theorem,
  we introduce a new sequence $\{ \varepsilon_n \}_{n \in \N}$ defined by
  \[
    \varepsilon_n := \epsilon_\ell
    \text{ for the unique } \ell \in \N \text{ satisfying } k_{\ell - 1} < n \leq k_\ell.
  \]
  Then $\lim_{n \to \infty} \varepsilon_n = 0$.
  Thus, by the case from the beginning of the proof, there exists $f \in X$
  such that $E(f, \mathcal{R}_n^1(I))_{X} \geq \varepsilon_n$ for all $n$.
  In particular,
  \[
    E(f, A_i)_{X}
    =    E(f, \mathcal{R}_{n_i}^1(I))_{X}
    \geq E(f, \mathcal{R}_{k_i}^1(I))_{X}
    \geq \varepsilon_{k_i}
    =    \epsilon_i
    \text{ for all } i \in \mathbb{N}.
  \]
  This completes the proof.
\end{proof}

For splines, we will use the following result.

\begin{theorem}\label{spline1d}
  Let $\mathcal{S}_{n,r}(I)$ denote the set of polynomial splines of degree $\leq n$
  with $r$ free knots in the interval $I = [a,b]$, $a < b$.
  Let $\{ r_i \}_{i \in \N}$ and $\{ n_i \}_{i \in \N}$ be two sequences of natural numbers.
  Define $A_0 := \{ 0 \}$ and ${A_i := \mathcal{S}_{n_i,r_i} \cap C(I)}$ for $i \in \N$.
  Finally, let either $X = C(I)$ or $X = L^q(I)$ where $q \in (0, \infty)$.

  Then for any null-sequence $\{ \eps_i \}_{i \in \N}$, there is a function $f \in X$
  satisfying
  \[
    E(f, A_i)_X \geq \eps_i \qquad \text{for all } i \in \N .
  \]
\end{theorem}

\begin{proof}
  Define $R_i := i + \max_{1 \leq j \leq i} r_j$ and $N_i := i + \max_{1 \leq j \leq i} n_j$
  for $i \in \N$, noting that both sequences $\{ R_i \}_{i \in \N}$ and $\{ N_i \}_{i \in \N}$
  are strictly increasing and that $R_i \geq r_i$ and $N_i \geq n_i$ for all $i \in \N$,
  which easily implies that $B_i := \mathcal{S}_{N_i, R_i} \cap X \supset A_i$ for all $i \in \N$.

  Then \cite[Theorem~6.12]{AO1} shows that there is $f \in X$ satisfying
  $E(f, B_i)_X \geq \eps_i$ for all $i \in \N$.
  Because of $A_i \subset B_i$ and hence $E(f, A_i)_X \geq E(f, B_i)_X$, this implies the claim.
\end{proof}

We are now in a position to prove our negative result
for approximation by feedforward neural networks with many layers.
In the following, we always assume that the activation function $\sigma$
is either a continuous spline $\sigma : \mathbb{R} \to \mathbb{R}$
or a rational function $\sigma(t) = p(t) / q(t)$
with univariate polynomials $p, q \in \mathbb{R}[t]$ and $q(t) \neq 0$ for all $t \in \R$.
Then, we consider the approximation of continuous learning functions,
defined on a compact subset of $\mathbb{R}^d$,
by neural networks using the activation function $\sigma$.

We denote by $\tau_{r, n}^{\sigma, d}$ the set of functions of $d$ real variables
computed by a feedforward neural network with at most $r$ hidden layers
and $n$ units in each layer, with activation function $\sigma$.
For example, an element of $\tau_{1,n}^{\sigma,d}$ is a function of the form
\[
  \phi(x)
  = \sum_{i=1}^n c_i \, \sigma(w^i\cdot x+b_i)
  \text{ where } w^i, x \in \mathbb{R}^d
  \text{ and } c_i, b_i \in \mathbb{R} ,
\]
and an element of $\tau_{2, n}^{\sigma, d}$ is either in $\tau_{1,n}^{\sigma,d}$,
or a function of the form
\[
 \phi(x)
 = \sum_{i=1}^n
     d_i \, \sigma \left(
                     \sum_{j=1}^n
                       c_{i,j} \, \sigma(w^{i,j} \cdot x + b_{i,j}) + \delta_i
                   \right),
  \text{ where } w^{i,j}, x \in \mathbb{R}^d
  \text{ and } c_{i,j}, b_{i,j}, d_i, \delta_i \in\mathbb{R} .
\]

The following is our main result concerning uniform approximation using
neural networks with more than one hidden layer.

\begin{theorem}\label{dimsup}
  Let $K \subset \R^d$ be a compact and convex set with at least two elements.
  Let $\{ r_k \}_{k \in \N}$ and $\{ n_k \}_{k \in \N}$ be arbitrary sequences of natural numbers,
  and let $\{ \varepsilon_k \}_{k \in \N}$ be an arbitrary sequence of real numbers
  converging to zero.
  Then:
  \begin{itemize}
    \item[$(a)$] If $\sigma(t) = p(t) / q(t)$ is a univariate rational function with $q(t) \neq 0$
                 for all $t \in \R$, then there exists $f \in C(K)$ such that
                 \[
                   E \big( f, \tau_{r_k, n_k}^{\sigma,d} \big)_{C(K)}
                   \geq \varepsilon_k
                   \qquad \text{ for all } k \in \N .
                 \]

    \item[$(b)$] If $\sigma : \R \to \R$ is continuous and piecewise polynomial with finitely
                 many pieces, then there exists $f \in C(K)$ such that
                 \[
                   E \big( f,\tau_{r_k,n_k}^{\sigma,d} \big)_{C(K)}
                   \geq \varepsilon_k
                   \qquad \text{ for all } k \in \N .
                 \]
                 In particular, this result applies to networks
                 with activation functions ${\sigma = \mathrm{ReLU}}$ or $\sigma = \mathrm{Hard\,Tanh}$,
                 which both are continuous piecewise linear functions with finitely many pieces.
  \end{itemize}
\end{theorem}

\begin{proof}
  \emph{(a)} Since $\# K \geq 2$, there are $a,b \in \R^d$ with $a \neq 0$
  such that $b, b + a \in K$.
  By convexity of $K$, this means that $\psi : [0,1] \to K, t \mapsto b + t \, a$
  is well-defined and continuous.
  Since $\sigma$ is a continuous rational function, it is not hard to see that
  for arbitrary $r,n \in \N$, there is some $\omega (r,n) = \omega(r,n,\sigma) \in \N$
  such that if $R \in \tau_{r,n}^{\sigma,d}$ is arbitrary,
  then $R \circ \psi$ is a univariate, continuous rational function of degree at most $\omega(r,n)$.
  Now, for $\{ \varepsilon_k \}_{k \in \N}$ as in the statement of the current theorem,
  Theorem~\ref{rational} yields a continuous function $g \in C([0,1])$ such that
  $E \big(g, \mathcal{R}_{\omega(r_k, n_k)}^1 ([0,1]) \big)_{C([0,1])} \geq \varepsilon_k$
  for all $k \in \N$.
  Clearly, one can extend $g$ to a continuous function $\overline{g} \in C(\R)$.
  Now, define
  \(
    f : \R^d \to \R, x \mapsto \overline{g} \big( (x - b) \cdot a / \| a \|^2 \big)
  \)
  and note that $f$ is continuous and satisfies
  \(
    f(\psi(t))
    = \overline{g} \big( t a \cdot a / \| a \|^2 \big)
    = \overline{g} (t)
  \)
  for all $t \in [0,1]$; that is, $f \circ \psi = g$.

  Now, for $R \in \tau_{r_k, n_k}^{\sigma, d}$, we saw above that $R \circ \psi$
  is a continuous univariate rational function of degree at most $\omega(r_k, n_k)$,
  i.e., $R \circ \psi \in \mathcal{R}_{\omega(r_k, n_k)}^1 ([0,1])$.
  Since $\mathrm{range}(\psi) \subset K$ and $f \circ \psi = g$, this implies
  \[
    \| f - R \|_{C(K)}
    \geq \| f \circ \psi - R \circ \psi \|_{C([0,1])}
    \geq E \big( g, \mathcal{R}_{\omega(r_k, n_k)}^1 ([0,1]) \big)_{C([0,1])}
    \geq \varepsilon_k .
  \]
  Since $R \in \tau_{r_k, n_k}^{\sigma, d}$ was arbitrary, this proves that
  $E(f, \tau_{r_k, n_k}^{\sigma, d})_{C(K)} \geq \varepsilon_k$ for all $k \in \N$.

  \medskip{}

  \emph{(b)}
  For $a, b \in \R^{d}$, let us define $\phi_{a,b} : \R \to \R^{d}, t \mapsto b + ta$
  and $K_{a,b} := \phi_{a,b}^{-1} (K)$.
  Note that $K_{a,b} \subset \R$ is either empty, or a closed interval,
  which is compact if $a \neq 0$.
  Next, let
  \[
    \mathcal{S}_{n,k}^{d,\mathrm{slice}}(K)
    \!:=\!
    \left\{
      f : K \to \R
      \colon
      f \text{ continuous and } f \circ \phi_{a,b}|_{K_{a,b}} \in \mathcal{S}_{n,k}(K_{a,b})
      \quad \text{for all } a, b \in \R^{d}
    \right\} ,
  \]
  where we interpret the condition $f \circ \phi_{a,b}|_{K_{a,b}} \in \mathcal{S}_{n,k}(K_{a,b})$
  as true in case of $K_{a,b} = \varnothing$.

  Since $\sigma : \R \to \R$ is continuous and piecewise polynomial with finitely many pieces,
  it follows from \cite[Lemma~3.6]{T2} that for each $n, r \in \N$
  there are ${N(n, r) = N ( \sigma, n, r ) \in \N}$
  and $M ( n, r ) = M ( \sigma, n, r ) \in \N$ such that
  \[
    \left\{
      f|_{K}\colon f\in\tau_{r,n}^{\sigma,d}
    \right\}
    \subset \mathcal{S}_{N ( n, r ), M ( n, r )}^{d,\mathrm{slice}} (K).
  \]
  Since $K$ is convex with $\# K \geq 2$, there are $a, b \in \R^{d}$ with $a \neq 0$ and such that
  ${K_{a,b} = [\alpha, \beta] =: I}$ for certain $\alpha, \beta \in \R$ with $\alpha < \beta$.

  Define $N_k := N(n_k, r_k)$ and $R_k := M(n_k, r_k)$.
  By Theorem~\ref{spline1d}, there is a function $f_{0} \in C(I)$ such that
  \[
    E \big( f_{0}, \mathcal{S}_{N_k, R_k} (I) \big)_{C(I)}
    \geq \varepsilon_{k}
    \qquad \text{ for all } \, k \in \N.
  \]
  By extending $f_{0}$ to be constant on $(-\infty, \alpha)$ and on $(\beta, \infty)$,
  we obtain a continuous function $f_{1} \in C(\R)$ satisfying $f_{1}|_{I} = f_{0}$.

  Define $f : K \to \R, x \mapsto f_{1} \left( a \cdot (x - b) / \|a\|^{2} \right)$,
  and note for $t \in I = K_{a,b}$ that
  \[
    \left(f\circ\phi_{a,b}\right)\left(t\right)
    = f_{1} \left( a \cdot ( b + a t - b) / \|a\|^{2} \right)
    = f_{1}(t)
    = f_{0}(t).
  \]
  Now, for arbitrary $g \in \tau_{r_{k}, n_{k}}^{\sigma,d}$, we have
  \(
    g|_{K}
    \in \mathcal{S}_{N(n_{k},r_{k}), M(n_{k}, r_{k})}^{d, \mathrm{slice}}(K)
    =   \mathcal{S}_{N_{k}, M_{k}}^{d, \mathrm{slice}}(K)
  \)
  and hence $g|_{K} \circ \phi_{a,b} \in \mathcal{S}_{N_{k}, M_{k}}(I)$.
  Since $\phi_{a,b}(I) \subset K$, we thus get
  \[
    \| f - g \|_{C(K)}
    \!\geq\! \left\Vert f \circ \phi_{a,b} - g \circ \phi_{a,b} \right\Vert_{C(I)}
    \!=\!    \left\Vert f_{0}-\left(g|_{K}\circ\phi_{a,b}\right)\right\Vert_{C(I)}
    \!\geq\! E \left(f_{0}, \mathcal{S}_{N_{k}, M_{k}}(I)\right)_{C(I)}
    \!\geq\! \varepsilon_{k},
  \]
  and hence $E(f, \tau_{r_{k}, n_{k}}^{\sigma, d})_{C(K)} \geq \varepsilon_{k}$ for all $k \in \N$.
\end{proof}

\begin{remark}\label{rem:UniformResultForSetsWithNonemptyInterior}
  A slight modification of the proof of Theorem~\ref{dimsup} shows
  that the very same statement holds true as soon as the compact (but not necessarily convex)
  set $K$ satisfies that $\{ b + t \, a \,\,\colon t \in [0,1] \} \subset K$ for certain
  $a,b \in \R^d$ with $a \neq 0$.
  This in particular holds for any compact set $K$ with nonempty interior.
\end{remark}

A similar lower bound as in Theorem~\ref{dimsup} also holds for approximation in $L^q$.

\begin{theorem}\label{thm:LowerBoundRationalSplineLq}
  Let $K \subset \R^d$ be measurable with nonempty interior and let $q \in (0,\infty)$.
  Assume that either $\sigma : \R \to \R$ is continuous and
  piecewise polynomial with finitely many pieces, or that $\sigma(t) = p(t) / q(t)$
  is a rational function with $q(t) \neq 0$ for all $t \in \R$.

  Then, for arbitrary sequences $\{ r_k \}_{k \in \N}$ and $\{ n_k \}_{k \in \N}$ of natural numbers,
  and any null-sequence $\{ \eps_k \}_{k \in \N} \subset \R$, there exists a function
  $f \in L^q(K)$ satisfying
  \[
    E\bigl(f, \tau_{r_k,n_k}^{\sigma,d}\bigr)_{L^q(K)}
    \geq \eps_k
    \qquad \text{for all } k \in \N .
  \]
\end{theorem}

\begin{proof}
  Because $K$ has nonempty interior, there exist $x^{(0)} \in K$ and $\delta > 0$
  satisfying ${K_0 := x^{(0)} + [-\delta,\delta]^d \subset K}$.
  For brevity, set $I := \bigl[x^{(0)}_1 - \delta, x^{(0)}_1 + \delta\bigr]$.
  Now, we distinguish two cases:

  \smallskip{}

  \noindent
  \textbf{Case~1 ($\sigma$ is rational):}
  In this case, it is easy to see that for arbitrary $r,n \in \N$ there exists
  $\omega(r,n) = \omega(n,r,\sigma) \in \N$ such that for any $y \in \R^d$
  and any $R \in \tau_{r,n}^{\sigma,d}$, the function $t \mapsto R(y + t \, e_1)$
  is a continuous rational function of degree at most $\omega(r,n)$,
  where $e_1 = (1,0,\dots,0) \in \R^d$ is the first standard basis vector.
  Next, Theorem~\ref{rational} provides a function $g_0 \in L^q(I)$ satisfying
  \[
    E\bigl(g_0, \mathcal{R}^1_{\omega(r_k,n_k)} (I)\bigr)_{L^q(I)}
    \geq \eps_k / (2\delta)^{(d-1)/q}
    \qquad \text{for all } k \in \N .
  \]
  Extend $g_0$ by zero to a function $g \in L^q(\R)$, and define
  $f : \R^d \to \R, x \mapsto \indicator_{K_0}(x) \cdot g(x_1)$,
  where Fubini's theorem easily shows that $f \in L^q(\R^d)$.

  Now, given any $y \in K_0$ with $y_1 = x^{(0)}_1$ and any $R \in \tau_{r_k, n_k}^{\sigma,d}$,
  recall that $t \mapsto R (y + t \, e_1)$ is a continuous rational function of degree at most
  $\omega(r_k, n_k)$, so that
  \begin{align*}
    \| f(y + t \, e_1) - R (y + t \, e_1) \|_{L^q_t ([-\delta,\delta])}
    & = \| g(y_1 + t) - R (y + t \, e_1) \|_{L^q_t ([-\delta,\delta])} \\
    & \geq E\bigl(g, \mathcal{R}_{\omega(r_k,n_k)}^1 (I)\bigr)_{L^q(I)}
      \geq \eps_k / (2\delta)^{(d-1)/q} .
  \end{align*}
  Therefore, writing $y^{(s')} := x^{(0)} + (0,s_2,\dots,s_d) \in K_0$
  for $s' = (s_2,\dots,s_d) \in [-\delta,\delta]^{d-1}$,
  Fubini's theorem shows that
  \begin{equation}
    \begin{split}
        \| f - R \, \|_{L^q(K)}^q
      & \geq \int_{[-\delta,\delta]^d}
               |f(x^{(0)} + s) - R(x^{(0)} + s)|^q
             \, d s \\
      & =    \int_{[-\delta,\delta]^{d-1}}
               \big\|
                 f(y^{(s')} + t \, e_1) - R(y^{(s')} + t \, e_1)
               \big\|_{L^q_t ([-\delta,\delta])}^q
             \, d s' \\
      & \geq (2\delta)^{d-1} \cdot \frac{\eps_k^q}{(2 \delta)^{d-1}}
        =    \eps_k^q \,\,,
    \end{split}
    \label{eq:FubiniArgument}
  \end{equation}
  and hence $E(f, \tau_{r_k,n_k}^{\sigma,d})_{L^q(K)} \geq \eps_k$ for all $k \in \N$.

  \medskip{}

  \noindent
  \textbf{Case~2 ($\sigma$ is a spline):}
  In this case, \cite[Lemma~3.6]{T2} shows for arbitrary $n,r \in \N$ that there are
  $N (n,r) = N(\sigma,n,r) \in \N$ and $M (n,r) = M(\sigma,n,r) \in \N$ such that
  for each $R \in \tau_{r,n}^{\sigma,d}$ and arbitrary $y \in \R^d$, we have
  \(
    \bigl(
      t \mapsto R(y + t \, e_1)
    \bigr)
    \in \mathcal{S}_{N(n,r), M(n,r)} \cap C(\R)
  \).
  Set $N_i := N(n_i, r_i)$ and $M_i := M(n_i, r_i)$, as well as $A_i := \mathcal{S}_{N_i,R_i} \cap C(I)$
  for $i \in \N$ and $A_0 := \{ 0 \}$.
  Then Theorem~\ref{spline1d} yields a function $g_0 \in L^q(I)$ satisfying
  $E(g_0, A_i)_{L^q(I)} \geq \eps_i / (2 \delta)^{(d-1)/q}$ for all $i \in \N$.
  As in the previous step, extend $g_0$ by zero to a function $g \in L^q(\R)$,
  and define $f : \R^d \to \R, x \mapsto \indicator_{K_0}(x) \cdot g(x_1)$,
  noting that $f \in L^q(\R^d)$ as a consequence of Fubini's theorem.

  Now, recall that $\bigl(t \mapsto R(y + t \, e_1)\bigr) \in \mathcal{S}_{N_k,M_k} \cap C(I) = A_k$
  for any $y \in K_0$ with $y_1 = x^{(0)}_1$ and any $R \in \tau_{r_k,n_k}^{\sigma,d}$,
  and hence
  \begin{align*}
    \| f(y + t \, e_1) - R (y + t \, e_1) \|_{L^q_t ([-\delta,\delta])}
    & = \| g(y_1 + t) - R (y + t \, e_1) \|_{L^q_t ([-\delta,\delta])} \\
    & \geq E(g, A_k)_{L^q(I)}
      \geq \eps_k / (2\delta)^{(d-1)/q} .
  \end{align*}
  Using this estimate, Equation~\eqref{eq:FubiniArgument} shows exactly as above that
  $E(f, \tau_{r_k,n_k}^{\sigma,d})_{L^q(K)} \geq \eps_k$.
\end{proof}

\vspace*{0.8cm}

\noindent{\sc J.~M.~Almira, P.~E.~Lopez-de-Teruel, D.~J.~Romero-L\'{o}pez}\\
Departamento de Ingenier\'{\i}a y Tecnolog\'{\i}a de Computadores,  Universidad de Murcia.\\
 30100 Murcia, SPAIN\\
e-mail: \texttt{jmalmira@um.es},  \texttt{pedroe@um.es},  \texttt{dj.romerolopez@um.es}

\vspace*{0.8cm}

\noindent  {\sc F.~Voigtlaender }\\
Department of Scientific Computing, Catholic University of Eichstätt-Ingolstadt. \\
85072 Eichstätt, GERMANY\\
e-mail: \texttt{felix@voigtlaender.xyz}

\begin{thebibliography}{99}

\bibitem{AO1}
{\sc J.~M.~Almira and T.~Oikhberg,}
Approximation schemes satisfying Shapiro's Theorem,
{\it J.~Approx.~Theory},
{\bf 164} (2012) 534--571.

\bibitem{AO2}
{\sc J.~M.~Almira and T.~Oikhberg,}
Shapiro's theorem for subspaces,
{\it J.~Math.~Anal.~Appl.},
{\bf 388} (2012) 282--302.

%

\bibitem{BarronL2Approximation}
{\sc A.~R.~Barron,}
Universal approximation bounds for superpositions of a sigmoidal function,
{\it IEEE~Trans.~Inform.~Theory},
{\bf 39} (1993) 930--945.

\bibitem{BarronUniformApproximation}
{\sc A.~R.~Barron,}
Neural net approximation,
in {\it Proc. 7th Yale Workshop on Adaptive and Learning Systems},
{\bf 39} (1992), Vol.~1, 69--72.


\bibitem{BianchiniBettiNumbers}
{\sc M.~Bianchini and F.~Scarselli,}
On the {C}omplexity of {N}eural {N}etwork {C}lassifiers:
{A} {C}omparison {B}etween {S}hallow and {D}eep {A}rchitectures,
{\it IEEE Trans.~Neural Netw.~Learn.~Syst.},
{\bf 25} (2014) 1553--1565.


\bibitem{EvansPDE}
{\sc L.~C.~Evans,}
\textit{Partial differential equations. Second edition,}
American Mathematical Society, Providence, RI, 2010.


\bibitem{FollandRA}
{\sc G.~B.~Folland,}
\textit{Real analysis. Second edition,}
John Wiley \& Sons, Inc., New York, 1999.


\bibitem{GMMR}
{\sc Y.~Gordon, V.~Maiorov, M.~Meyer, and S.~Reisner,}
On the best approximation by ridge functions in the uniform norm,
\textit{Constr. Approx.},
{\bf 18} (2002) 61--85.

\bibitem{GI}
{\sc N.~J.~Guliyev and V.~E.~Ismailov,}
On the approximation by single hidden layer feedforward neural networks with fixed weights,
\textit{Neural Netw.},
{\bf 98} (2018) 296--304.

\bibitem{GI2}
{\sc N.~J.~Guliyev and V.~E.~Ismailov,}
Approximation capability of two hidden layer feedforward neural networks with fixed weights,
\textit{Neurocomputing},
{\bf 316} (2018) 262--269.

\bibitem{HSW}
{\sc K.~Hornik, M.~Stinchcombe, and H.~White,}
Multilayer feedforward networks are universal approximators,
\textit{Neural Netw.}
{\bf 2} (1989) 359--366.


\bibitem{IsmailovBoundedNumberOfNeurons}
{\sc V.~E.~Ismailov,}
On the approximation by neural networks with bounded number of neurons in hidden layers,
\textit{J. Math. Anal. Appl.}
{\bf 417} (2014) 963--969.


\bibitem{KainenDependenceOnInputDimension}
{\sc P.C.~Kainen, V.~K\r{u}rkov\'{a}, and M.~Sanguineti,}
Dependence of {C}omputational {M}odels on {I}nput {D}imension:
{T}ractability of {A}pproximation and {O}ptimization {T}asks,
\textit{IEEE Trans. Inform. Theory}
{\bf 58} (2012) 1203--1214.


\bibitem{KarpinskiVCDimensionPfaffian}
{\sc M.~Karpinski and A.~Macintyre,}
Polynomial bounds for VC dimension of sigmoidal and general Pfaffian neural networks,
\textit{J.~Comput.~Syst.~Sci.}
{\bf 54} (1997) 169--176.

\bibitem{K1}
{\sc V.~K\r{u}rkov\'{a},}
Kolmogorov's theorem and multilayer neural networks,
\textit{Neural Netw.}
{\bf 5} (1992) 501--506.

\bibitem{K2}
{\sc V.~K\r{u}rkov\'{a},}
Kolmogorov's theorem is relevant,
\textit{Neural Comput.}
{\bf 3} (1991) 617--622.

%

\bibitem{KurkovaHighlyVarying}
{\sc V.~K\r{u}rkov\'{a} and M.~Sanguineti,}
Model complexities of shallow networks representing highly varying functions,
\textit{Neurocomputing}
{\bf 171} (2016) 598--604.

\bibitem{KurkovaProbabilisticLowerBounds}
{\sc V.~K\r{u}rkov\'{a} and M.~Sanguineti,}
Probabilistic lower bounds for approximation by shallow perceptron networks,
\textit{Neural Netw.}
{\bf 91} (2017) 34--41.

\bibitem{LLPS}
{\sc M.~Leshno,  V.~Y.~Lin, A.~Pinkus, and S.~Schocken,}
Multilayer Feedforward Networks with a non polynomial activation function can approximate any function,
\textit{Neural Netw.}
{\bf 6} (1993) 861--867.

\bibitem{LGM}
{\sc G.~G.~Lorentz, M.~V.~Golitschek, and Y.~Makovoz,}
\textit{Constructive Approximation: Advanced Problems,}
Springer, 1996.

\bibitem{MP}
{\sc V.~Maiorov and A.~Pinkus,}
Lower bounds for approximations by MLP neural networks,
\textit{Neurocomputing}
{\bf 25} (1999) 81--91.


\bibitem{MaiorovRidgeFunctionSobolevDistance}
{\sc V.~E.~Maiorov,}
Best approximation by ridge functions in $L_p$-spaces,
\textit{Ukra\"{\i}n.~Mat.~Zh.}
{\bf 62} (2010) 396--408.

%

\bibitem{MhaskarOptimalApproximationShallowNets}
{\sc H.~N.~Mhaskar,}
Neural networks for optimal approximation of smooth and analytic functions,
\textit{Neural Comput.}
{\bf 8} (1996) 164--177.

\bibitem{MhaskarTractabilityOfIntegration}
{\sc H.~N.~Mhaskar,}
On the tractability of multivariate integration and approximation by neural networks,
\textit{J.~Complexity}
{\bf 20} (2004) 561--590.


\bibitem{SRK}
{\sc K-Y.~Siu, V.~P.~Roychowdhury, and T.~Kailath,}
Rational Approximation Techniques for Analysis of Neural Networks,
\textit{IEEE Trans.~Inf.~Theory},
{\bf 40} (1994) 455--466.

\bibitem{T}
{\sc M.~Telgarsky,}
Neural networks and rational functions,
Proc.~Machine Learning Research ICML (2017) 3387--3393.

\bibitem{T2}
{\sc M.~Telgarsky,}
Benefits of depth in neural networks,
JMLR: Workshop and Conference Proceedings
{\bf 49} (2016) 1--23.

\bibitem{W}
{\sc R.~C.~Williamson,}
Rational parametrization of neural networks,
\textit{Adv. Neural Inf. Process Syst.}
{\bf 6} (1993) 623--630.

\bibitem{WB}
{\sc R.~C.~Williamson and P.~L.~Barlett,}
Splines, rational functions and neural networks,
\textit{Adv.~Neural Inf.~Process Syst.}
{\bf 5} (1992) 1040--1047.



\end{thebibliography}
\end{document}